\documentclass[sigconf]{acmart}

\AtBeginDocument{%
  }
    
\copyrightyear{2025} 
\acmYear{2025} 
\setcopyright{acmlicensed}
\acmConference[KDD '25]{Proceedings of the 31st ACM SIGKDD Conference on Knowledge Discovery and Data Mining V.2}{August 3--7, 2025}{Toronto, ON, Canada} \acmBooktitle{Proceedings of the 31st ACM SIGKDD Conference on Knowledge Discovery and Data Mining V.2 (KDD '25), August 3--7, 2025, Toronto, ON, Canada} \acmDOI{10.1145/3711896.3736923} 
\acmISBN{979-8-4007-1454-2/2025/08}

\usepackage{url}
\usepackage{enumitem}
\usepackage{graphicx}
\usepackage{multirow}
\usepackage{booktabs}
\usepackage{wrapfig}
\usepackage{amsthm,amsmath}
\usepackage{diagbox}
\usepackage{float}
\usepackage{algorithm}
\usepackage{algorithmic}
\newtheorem{proposition}{Proposition}
\usepackage{listings}
\usepackage{xcolor}      
\usepackage{cleveref}    
\usepackage{hyperref}  
\usepackage{subfig}
\usepackage{fancyhdr}
\begin{document}

\title{Efficient Heuristics Generation for Solving Combinatorial Optimization Problems Using Large Language Models}

\author{Xuan Wu}
\affiliation{%
  \institution{College of Computer Science and Technology, Jilin University}
  \city{Changchun}
  \state{Jilin}
  \country{China}
}
\email{wuuu22@mails.jlu.edu.cn}

\author{Di Wang}
\affiliation{%
  \institution{LILY Research Centre, Nanyang Technological University}
  \country{Singapore}
}
\email{wangdi@ntu.edu.sg}

\author{Chunguo Wu}
\affiliation{%
  \institution{Key Laboratory of
 Symbolic Computation and Knowledge Engineering of Ministry of Education, Jilin University}
  \city{Changchun}
  \state{Jilin}
  \country{China}
}\email{wucg@jlu.edu.cn}

\author{Lijie Wen}
\affiliation{%
  \institution{School of Software, Tsinghua University}
  \city{Beijing}
  \country{China}
}\email{wenlj@tsinghua.edu.cn}

\author{Chunyan Miao}
\affiliation{%
  \institution{LILY Research Centre, Nanyang Technological University}
  \country{Singapore}
}\email{ascymiao@ntu.edu.sg}

\author{Yubin Xiao}
\authornote{Corresponding Authors.}
\affiliation{%
  \institution{College of Computer Science and Technology, Jilin University}
  \city{Changchun}
  \state{Jilin}
  \country{China}
}
\email{xiaoyb21@mails.jlu.edu.cn}

\author{You Zhou}
\authornotemark[1]
\affiliation{%
  \institution{College of Software, Jilin University}
  \city{Changchun}
  \state{Jilin}
  \country{China}
}\email{zyou@jlu.edu.cn}
\renewcommand{\shortauthors}{Xuan Wu et al.}

\begin{abstract}
Recent studies exploited Large Language Models (LLMs) to autonomously generate heuristics for solving Combinatorial Optimization Problems (COPs), by prompting LLMs to first provide search directions and then derive heuristics accordingly. However, the absence of task-specific knowledge in prompts often leads LLMs to provide unspecific search directions, obstructing the derivation of well-performing heuristics. Moreover, evaluating the derived heuristics remains resource-intensive, especially for those semantically equivalent ones, often requiring omissible resource expenditure. To enable LLMs to provide specific search directions, we propose the Hercules algorithm, which leverages our designed Core Abstraction Prompting (CAP) method to abstract the core components from elite heuristics and incorporate them as prior knowledge in prompts. We theoretically prove the effectiveness of CAP in reducing unspecificity and provide empirical results in this work. To reduce computing resources required for evaluating the derived heuristics, we propose few-shot Performance Prediction Prompting (PPP), a first-of-its-kind method for the Heuristic Generation (HG) task. PPP leverages LLMs to predict the fitness values of newly derived heuristics by analyzing their semantic similarity to previously evaluated ones. We further develop two tailored mechanisms for PPP to enhance predictive accuracy and determine unreliable predictions, respectively. The use of PPP makes Hercules more resource-efficient and we name this variant Hercules-P. Extensive experiments across four HG tasks, five COPs, and eight LLMs demonstrate that Hercules outperforms the state-of-the-art LLM-based HG algorithms, while Hercules-P excels at minimizing required computing resources. In addition, we illustrate the effectiveness of CAP, PPP, and the other proposed mechanisms by conducting relevant ablation studies.
\end{abstract}

\ccsdesc[500]{Mathematics of computing~Combinatorial optimization}
\ccsdesc[500]{Computing methodologies~Search methodologies}

\keywords{Large Language Models; Heuristic Generation; Combinatorial Optimization Problems}
\maketitle
\vspace{-0.3cm}
\newcommand\kddavailabilityurl{https://doi.org/10.5281/zenodo.15462797}

\ifdefempty{\kddavailabilityurl}{}{
\begingroup\small\noindent\raggedright\textbf{KDD Availability Link:}\\
The source code of this paper has been made publicly available at \url{\kddavailabilityurl}.
\endgroup
}
\vspace{-0.43cm}
\section{Introduction}
Heuristic algorithms have long been a preferred approach for solving Combinatorial Optimization Problems (COPs) \cite{REGO2011427, wu_incorporating_2023}. To automate the derivation of heuristics for a given COP, Heuristic Generation (HG) methods have attracted significant attention \cite{burke_hyper_2013}. Early HG methods predominantly employ Evolutionary Computation (EC) to derive heuristics. However, these methods focus on the exploration and exploitation in the micro search space composed of the predefined modules, resulting in limited performance \citep{ye2024reevo}.

Recently, the emergence of Large Language Models (LLMs) has facilitated the autonomous derivation of heuristics, eliminating the need for manually defining the search space  \citep{liu2023algorithm, liu2024evolution, vanstein2024llamea}. In addition, compared to conventional EC algorithms, LLMs benefit from a broader search space by leveraging their mega-size training corpora, resulting in elevated performance \citep{yang2024large, ma2024llam, liu2024large}. Specifically, these LLM-based HG methods exploit LLMs to provide search directions, which are then used to derive (novel) offspring heuristics \citep{romera_mathematical_2024}. These produced heuristics are subsequently evaluated using COP instances to determine their fitness values, with the better-performing heuristics carried over to the next iteration. For example, \citet{liu2023algorithm} proposed prompting methods that emulate crossover and mutation operators as search strategies, thereby implicitly providing search directions. To let LLMs offer more explicit search directions, \citet{ye2024reevo} proposed Reflection Prompting (RP), which requires LLMs to reflect on the relative performance of the produced heuristics and provide insights as search directions. These directions are then used to derive heuristics with expected elevated performance in subsequent crossover and mutation promptings.

\begin{figure}[!t]
\centering
\includegraphics[scale=0.45]{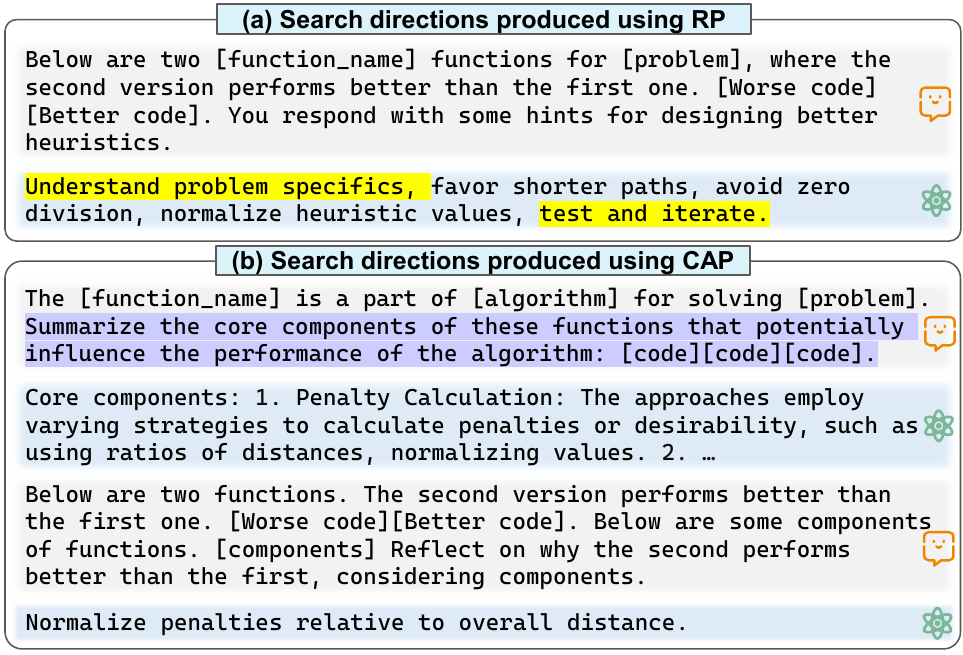}
	\caption{Illustration of the search directions produced using RP \citep{ye2024reevo} and CAP (our method) for the task described in Section~\ref{4.1}. When RP prompts LLMs (GPT-4o-mini used in this example) for search directions directly, the LLMs may respond with unspecific search directions (highlighted in yellow). Different from RP, CAP enhances the quality of the produced search directions by first prompting the LLMs to abstract the core components of elite heuristics as prior knowledge in a zero-shot manner (highlighted in purple).} 
	\label{figcap}
    \vspace{-0.8cm}
\end{figure}

These existing LLM-based HG methods face two key challenges. Firstly, when prompting LLMs to provide search directions (e.g., reflections on the relative performance of heuristics), the lack of task-specific knowledge in prompts often leads to over-generalized, unspecific directions that hinder the derivation of high-performance heuristics. As illustrated in Figure~\ref{figcap}(a), the produced search directions \textit{``Understand problem specifics"} and \textit{``test and iterate"} are vague, over-general, and lack actionable steps required for heuristic generation. Consequently, they contribute little to the derivation of high-performance heuristics. In contrast, other elements of the produced search directions are more specific. For example, \textit{``normalize heuristic values"} provides an actionable step that can be directly applied to derive heuristics. Therefore, it is essential to reduce unspecificity in the produced search directions. Secondly, during the search process, LLM-based HG methods often derive numerous heuristics, some of which may be semantically or even literally identical, as illustrated in Figure~\ref{code}. Reevaluating these heuristics using COP instances (i.e., conventional fitness evaluation method) not only wastes computing resources but also significantly prolongs the search process \citep{chen2024large}. In particular, these heuristics often involve numerous linear operations and conditional branches, which GPUs cannot efficiently accelerate \citep{7886331}. In addition, providing LLMs with all historical heuristics to avoid deriving semantically similar ones is impractical. This approach may compel LLMs to derive overly random or unviable heuristics, while significantly increasing the cost of context tokens.
\begin{figure}[!t]
\centering
\includegraphics[scale=0.45]{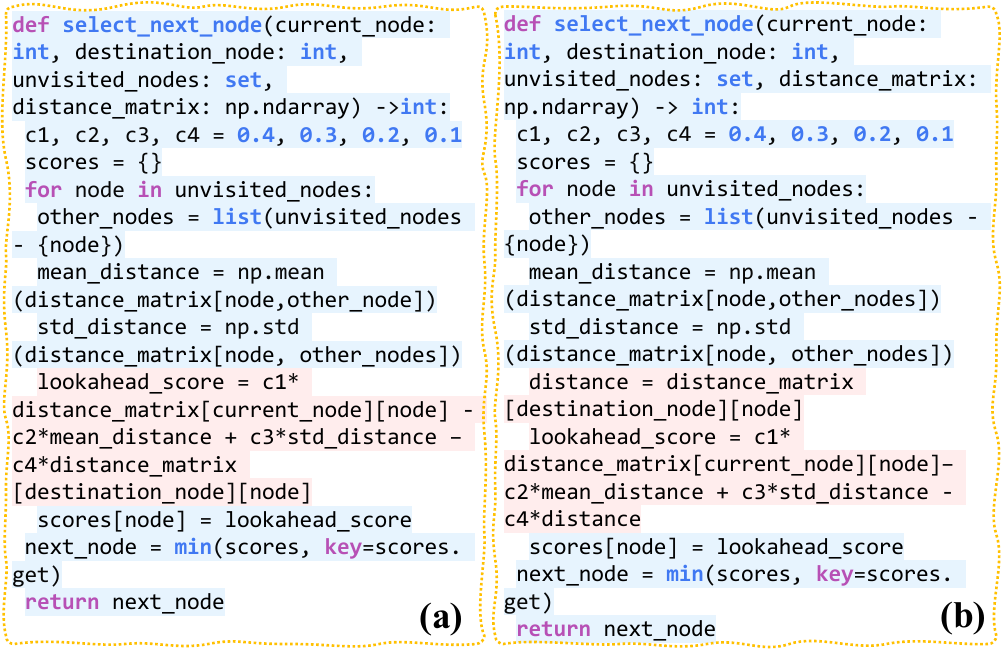}
	\caption{Illustration of two heuristics with identical semantics, produced by RP \citep{ye2024reevo} (GPT-3.5-turbo used in this example) for the task described in Section~\ref{4.2}. Code snippets with literal equivalence are highlighted in blue, while those with semantic equivalence are highlighted in pink.} 
	\label{code}
    \vspace{-0.8cm}
\end{figure}

To better address the first challenge, we propose \textbf{He}u\textbf{r}isti\textbf{c} Generation \textbf{U}sing \textbf{L}arge Languag\textbf{e} Model\textbf{s} (\textbf{Hercules}), which exploits our proprietary, straightforward yet effective Core Abstraction Prompting (\textbf{CAP}) method to reduce unspecificity in the produced search directions and thus enable the derivation of high-performance heuristics. Specifically, CAP directs an LLM to abstract the core components from the top-$k$ heuristics (i,e., elite heuristics) in the current population and then provide more specific search directions based on these components  (see Section~\ref{3.1}). Notably, as illustrated in Figure~\ref{figcap}(b), CAP operates in a zero-shot manner, abstracting the core components without providing any examples to guide this abstraction process, which leads to significant savings in context token costs. Meanwhile, by incorporating the concept of information gain, we theoretically prove that CAP can reduce unspecificity in the produced search directions in Appendix~\ref{appendixA}. To couple with CAP, we introduce a rank-based selection mechanism that increases the likelihood of selecting high-performance heuristics as parents (used in the following crossover and mutation promptings), rather than relying on random selection \citep{ye2024reevo}.

To better address the second challenge, we propose \mbox{ \textbf{Hercules-P}}, which integrates CAP with our novel Performance Prediction Prompting (\textbf{PPP}) method. PPP operates in a few-shot manner by presenting LLMs with a small set of previously evaluated heuristics as examples and prompting LLMs to predict the fitness values of the newly produced heuristics based on their semantic similarity to the presented examples (see Section~\ref{3.2}). Therefore, PPP reduces the number of heuristics that require evaluation using COP instances. Generally speaking, to enhance the predictive accuracy of PPP, we can either increase the number of examples or enhance their quality. However, collecting numerous heuristic examples along with their corresponding performance is resource-intensive. This contradicts to the primary purpose of incorporating PPP, which is to reduce resource expenditure during the search process. Moreover, unlike Neural Architecture Search (NAS), which benefits from extensive benchmarks \citep{ying_19_bench,QIU2023110671,ecgp}, the HG task lacks benchmarks with pre-evaluated heuristics. Therefore, we opt to provide higher-quality examples through a tailored example selection mechanism, termed EXEMPLAR, which favors distinct parent heuristics with superior performance as examples. Meanwhile, to determine unreliable predictions, we develop the Confidence Stratification (ConS) mechanism that requires the LLM to provide confidence levels for the predicted fitness values, thereby facilitating the identification of heuristics that need reevaluation. In summary, PPP reduces the resource expenditure in heuristic evaluations while maintaining population diversity, making it effective for tasks with a border search space. To the best of our knowledge, \textbf{our work proposes the first LLM-based performance predictor for the HG task.}

To assess the performance of the proposed Hercules and Hercules-P algorithms, we conduct extensive experiments on four HG tasks (see Section~\ref{4}). The experimental results demonstrate that Hercules outperforms the state-of-the-art (SOTA) LLM-based HG algorithms across four HG tasks, five COPs, and eight LLMs, without significantly increasing context or generation token costs. By incorporating PPP, Hercules-P significantly reduces the overall search time by 7\%$\sim$59\% when compared to Hercules, while achieving on-par performance on the gain metric. Finally, ablation studies validate the effectiveness of the proposed EXEMPLAR and ConS methods.

The key contributions of this work are as follows.

\textbf{i)} We propose the zero-shot CAP method, which reduces unspecificity in the LLM-produced search directions, enabling the derivation of high-performance heuristics. We also provide theoretical proof of CAP's effectiveness in reducing unspecificity by utilizing the concept of information gain.

\textbf{ii)} We propose the few-shot PPP method, a first-of-its-kind  LLM-based performance predictor specifically designed for HG tasks. PPP predicts the performance of newly produced heuristics by analyzing their semantic similarity to previously evaluated ones. Moreover, we develop two novel mechanisms: EXEMPLAR and ConS, which significantly enhance the overall performance of PPP.

\textbf{iii)} The experimental results demonstrate that our proposed Hercules achieves SOTA performance across four HG tasks, five COPs, and eight LLMs, while Hercules-P excels at reducing resource expenditure. Finally, ablation study results validate the effectiveness of all proposed methods.
\begin{figure*}[!t]
\centering
	\includegraphics[scale=0.9]{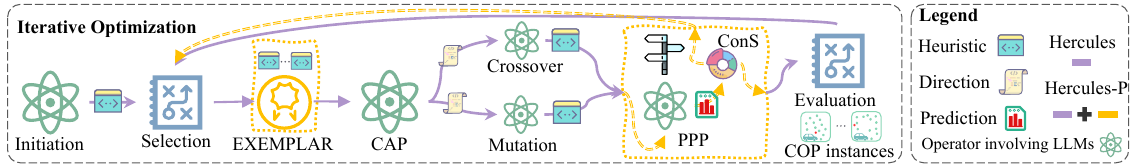}
	\caption{Overview of the proposed Hercules and Hercules-P algorithms. Hercules exploits CAP to provide specific search directions, which are then used to guide LLMs in deriving high-performance heuristics. In Hercules, the performance of all derived heuristics on a set of COP instances determines their respective fitness values. In contrast, Hercules-P evaluates only a subset of the produced heuristics with COP instances, while the rest are assessed using the proposed PPP method.} 
	\label{figframework}
        \vspace{-0.5cm}
\end{figure*}
\vspace{-0.2cm}
\section{Related Work}
In this section, we review the relevant literature.
\vspace{-0.4cm}
\subsection{LLM-based Heuristic Generation Algorithms}
\label{2.1}
Conventional EC-based HG algorithms search for the optimal combination of the predefined heuristic modules \citep{Keller_Linear_2007}, which often limits their performance. In contrast, LLM-based HG algorithms eliminate the need for predefining the search space, liberating researchers from manual customization and enabling the derivation of high-performance heuristics \citep{ wu2024evolutionary, huang2024large}. Specifically, these algorithms begin with a seed heuristic to prompt LLMs to derive multiple heuristics as the initial population \citep{liu2023algorithm, liu2024evolution, ye2024reevo}. Each heuristic is then evaluated using a set of COP instances, with its performance serving as its fitness value. During the iterative process, certain heuristics are selected as parents and presented to LLMs to derive (novel) offspring heuristics. This approach emulates the concepts of crossover and mutation, while implicitly providing search directions for the LLMs to derive heuristics. In addition, certain studies exploit LLMs to provide explicit search directions for deriving well-performing heuristics \citep{ye2024reevo}. However, these LLM-based HG algorithms overlook the issue of unspecificity in LLM responses (see Figure~\ref{figcap}(a)), which can lead to unspecific search directions that do not contribute to discovering high-performance heuristics.

Similar challenges are observed in tasks such as arithmetic and symbolic reasoning, making it crucial to evoke LLM reasoning through a multi-step process and incorporate task-specific knowledge \citep{yu2024thought, jiang2024llm, lv2024coarsetofine}. For example, \citet{wei2022chain} proposed Chain-of-Thought (CoT) prompting, which directs LLMs to emulate the given examples in completing a multi-step solution process, leading to more accurate answers. Subsequently, \citet{zheng2024take} proposed the few-shot Step-back Prompting (SP), which exploits the given examples to enable LLMs to abstract high-level principles and then apply these principles in reasoning. In a similar multi-step fashion, we propose CAP to mitigate unspecificity in the produced search directions for better solving HG tasks. However, unlike CoT and SP, CAP operates in a zero-shot manner, by abstracting the core components without any examples to guide the abstraction process.

\vspace{-0.2cm}
\subsection{LLM-based Performance Prediction Methods}
\label{2.2}
In the field of NAS, performance predictors, typically Deep Neural Networks, are widely used to reduce search costs by predicting the performance of candidate architectures \citep{baker2017designing, wu_weaker}. These predictors model neural architectures as graphs, where nodes represent subnets and edges represent the connections between subnets \citep{Chu_2023_ICCV, NEURIPS2022_572aaddf}. The graphs are then encoded into vectors, and the mapping between these vectors and the corresponding performance metrics is learned. Recently,  \citet{chen2024large} and \citet{jawahar2024llm} proposed LLM-based predictors for predicting the performance of neural architectures. Specifically, they employed examples of architectures and corresponding performance metrics to prompt LLMs, leveraging semantic similarity to predict the performance of newly searched architectures.

In the context of HG, conventional performance predictors may struggle to accurately evaluate heuristics due to the difficulty in modeling these diverse and complex heuristics as graph structures. However, the LLM-based predictor presents a promising alternative by eliminating the need for explicit heuristic modeling. Consequently, this paper leverages LLMs to predict the performance of heuristics for effectively solving HG tasks. However, unlike \citep{chen2024large}  and \citep{jawahar2024llm}, which relied on a larger number of examples, our PPP emphasizes the use of only the higher-quality examples to improve predictive performance (see Section~\ref{3.2} for more details). 
\vspace{-0.3cm}
\subsection{Neural Combinatorial Optimization Solvers}
\label{2.3}
Neural Combinatorial Optimization (NCO) refers to a class of Neural Network solvers that either independently solve COPs or collaborate with heuristic algorithms \citep{bengio_machine_2021,9393606, wu_2024_survey,bogyrbayeva_learning_2022}. To enable the derivation of insights from historical COP instances and efficiently handle batches of instances in parallel, researchers have recently developed numerous NCO solvers \citep{liufeimulti, xiao_distilling_2024}. However, these NCO solvers still face several challenges. Two of the most prominent ones are how to improve their generalization capabilities \citep{zhou_towards_2023,xiao2024improving} and their performance on large-scale COPs \citep{hou_generalize_2023, sun_difusco_2023,  mingzhao}. Recently, \citet{wang2024distance} proposed a distance-aware heuristic algorithm designed to enhance the generalization ability of NCO solvers trained on small-scale COPs for solving large-scale COPs. To assess the effectiveness of the proposed Hercules and Hercules-P algorithms, we apply them to improve the performance of two classic NCO solvers on both small-scale and large-scale COPs in Section~\ref{4.4}.

\section{Hercules and Hercules-P}

The illustrations of Hercules and Hercules-P are schematically presented in Figure~\ref{figframework}. In this section, we first introduce CAP, which is designed to provide more specific search directions for deriving heuristics. We then prove that CAP can reduce unspecificity of the produced search directions. Finally, we present the design of PPP, along with tailored EXEMPLAR and ConS mechanisms.
\vspace{-0.2cm}
\subsection{Core Abstraction Prompting (CAP)}
\label{3.1}
As aforementioned, when LLMs are tasked with providing search directions, they often generate directions that lack specificity for heuristic derivation. As illustrated in the RP example in Figure~\ref{figcap}(a), certain directions, such as \textit{``Understand problem specifics"} and \textit{``test and iterate"}, lack relevance to heuristic derivation and fail to derive well-performing heuristics.

In this case and many others, providing prior knowledge in prompts can help LLMs reduce unspecificity in their responses, leading to more focused, specific search directions. To achieve this, we propose the zero-shot CAP method, which can abstract the core components from the top-$k$ heuristics in the current population without additional guidance. Because the core components are essential for heuristic performance \citep{7339682, liu2024evolution}, leveraging them enables LLMs to provide more specific search directions. As shown in Figure~\ref{figcap}(b), the suggested direction \textit{``Normalize penalties relative to overall distance"} may lead to more effective heuristic generation. In addition, CAP abstracts the core components once per iteration, instead of abstracting distinct components separately for crossover and elitist mutation operators. Consequently, this approach helps prevent a significant increase in context and generation token costs compared to RP (see Table~\ref{tabsearchreport} in the experiment section). 


In the field of information theory, the advantage of CAP can be quantified using the concept of information gain. \citet{hu2024uncertainty} defined information gain as the reduction in entropy between two states. Extending this concept, we use information gain to quantify entropy reduction in scenarios with and without abstraction, facilitating the assessment of CAP in reducing unspecificity. Specifically, the entropy without abstraction (i.e., the core components are not presented to LLMs) in the $t$th iteration is defined as follows:
\begin{equation}
H(\Omega_t) = -\sum\nolimits_{i: \omega_i \in \Omega_t} p(\omega_i | \Omega_t) \log p(\omega_i | \Omega_t),
\end{equation}
where $\omega_i$ denotes a direction belonging to the set of all possible directions $\Omega_t$.

When the core components are used as prior knowledge in prompts, an LLM can provide more specific, subdivided search directions either based on one of these core components or disregarding all core components. Consequently, the set of all possible directions, $\Omega_t$, can be partitioned into mutually exclusive subsets, $\Omega_j$, where $\bigcup_{j=0}^{k} \Omega_j = \Omega_t$. Here, when $j \in \{ 0, 1, \dots, k-1\}$, $\Omega_j$ denotes the subset of directions associated with the $j$th core component (for simplicity, we assume a one-to-one correspondence between core components and heuristics), while $j=k$ corresponds to the subset of directions independent of any core component.

Assuming that the produced direction belongs to the $j$th subset ($j \in \{ 0,1, \dots,k\}$) after providing the core components, the remaining entropy is defined as follows:
\begin{equation}
 H(\Omega_j) = -\sum\nolimits_{i: \omega_i \in \Omega_j} p(\omega_i | \Omega_j) \log p(\omega_i | \Omega_j).
\end{equation}
Then, the entropy with abstraction (i.e., the expected remaining entropy) is defined as $\sum\nolimits_{j=0}^{k} p_j H(\Omega_j)$, where $p_j$ denotes the probability that the search direction belongs to the $j$th subset, i.e., $p_j = p(\Omega_j) / p(\Omega_t)$. Thus, the information gain from abstracting the core components in the $t$th iteration (the entropy reduction without and with abstraction) is defined as follows:
\begin{equation}
    IG(\Omega_t) = H(\Omega_t) - \sum\nolimits_{j=0}^{k} p_j H(\Omega_j).
    \label{eq}
\end{equation}

As proven in Appendix~\ref{appendixA}, Eq. (\ref{eq}) simplifies to the following expression, whose value ranges within the $(0, \log{(k+1)}]$ interval:
\begin{align}
    IG(\Omega_t) =  - \sum\nolimits_{j=0}^{k} p_j \log p_j.
\end{align}
Therefore, in theory, providing the core components as prior knowledge in prompts can reduce unspecificity in LLM responses and yield more specific search directions, subsequently leading to heuristics with higher performance.

To fine-search the space with high-quality heuristics, we adopt a rank-based selection mechanism. Specifically, the probability of selecting the $i$th heuristic as a parent is computed as follows:
\begin{equation}
p(x_i) = \frac{1}{\text{rank}(x_i) + N} \bigg/\sum\nolimits_{j=1}^{N} \frac{1}{\text{rank}(x_j) + N},
\label{rank}
\end{equation} 
where $N$ denotes the population size, and $\text{rank}(\cdot)$ returns the rank of the associated fitness value in the ascending order. In addition, Hercules adopts the core components of the \mbox{top-$k$} heuristics as prior knowledge during the first $\lambda$ percent of iterations ($\lambda \in$ [0,1]). In the later iterations, following \citep{4812104,9344816}, to better preserve population diversity, Hercules directly applies the core components of the parent heuristics as prior knowledge to provide search directions, bypassing the abstraction process of elite heuristics.


\vspace{-0.2cm}
\subsection{Performance Prediction Prompting (PPP)}
\label{3.2}

\begin{figure}[!t]
\centering\includegraphics[scale=0.52]{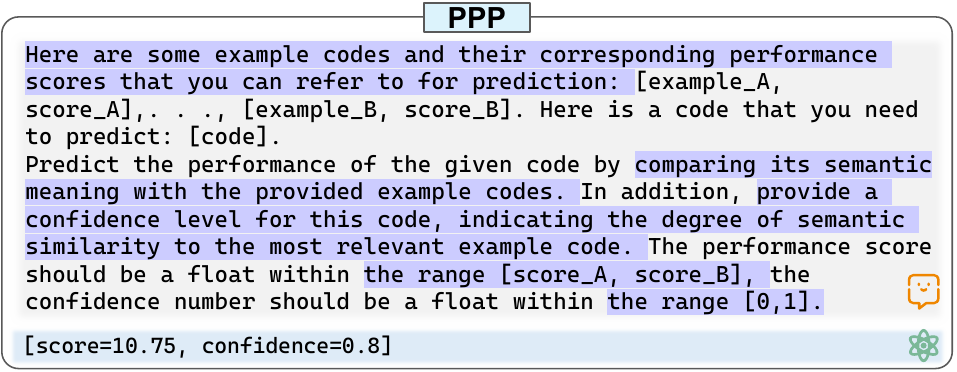}
	\caption{Illustration of the prediction process using the proposed PPP method. By analyzing the semantic similarity between the heuristics to be predicted and the previously evaluated ones, LLMs can respond with a performance score for each heuristic  with an associated confidence level.} 
	\label{figppp}
        \vspace{-0.6cm}
\end{figure}
Semantic features have demonstrated significant merits in software engineering tasks, e.g., identifying the defective code regions \citep{LIU2023111753}, due to their influence on the overall code performance. Motivated by this concept, we propose the few-shot PPP method, which leverages LLMs to predict the performance of newly produced heuristics by analyzing their semantic similarity to previously evaluated ones, as shown in Figure~\ref{figppp}. To achieve higher predictive accuracy with a small number of $N_e$ examples, we propose an example selection mechanism called EXEMPLAR, which operates on a principle similar to providing a more relevant, well-defined knowledge base in retrieval-augmented generation \citep{gao2023retrieval}. Specifically, EXEMPLAR selects the historically best and worst heuristics, i.e., $x_{lb}$ and $x_{ub}$, respectively, as prediction boundaries (assuming the goal of the HG task is to derive the heuristic with the minimum fitness value), while prioritizing parent heuristics with better performance (i.e., lower fitness value). Parent heuristics with better performance are typically more complex and richer in semantic features than those with inferior performance, highly likely leading to higher prediction accuracy. In addition, any heuristic with the same fitness value as a previously selected example will not be chosen as an example. Because if LLMs encounter multiple examples sharing the same fitness value, their predictions may become biased towards this common fitness value, potentially overlooking semantic features. If each example has a distinct fitness value, LLMs can more effectively leverage semantic features to predict the performance of the new heuristics. The set of examples $\mathcal{P}_e$ is selected as follows:
\begin{equation}
\begin{aligned}
 \mathcal{P}_e & = \{ x_{lb}, x_{ub}  \mid x_{lb} = \arg\min_{x \in \mathcal{P}_h} f(x), x_{ub} = \arg\max_{x \in \mathcal{P}_h} f(x) \} \\
 & \cup  \{x \mid \mathop{\operatorname{arg\,top(N_e-2)}}\limits_{x \in \mathcal{P}_t} f(x) \}, \\
 & \mathcal{P}_t = \{ x \in \mathcal{P}_p \setminus \{x_{lb}, x_{ub}\} \mid f(x_i) \neq f(x_j), \forall i \neq j  \},
\end{aligned}
\label{EXEMPLAR}
\end{equation}
where $\mathcal{P}_h$ and $\mathcal{P}_p$ denote the set of all historical heuristics and the set of parent heuristics selected from the current iteration according to Eq. (\ref{rank}) to produce offspring, respectively, and $f(\cdot)$ denotes the fitness evaluation function, introduced in the following paragraph. EXEMPLAR selects the set $\mathcal{P}_e$ for each iteration. 



Nevertheless, LLMs cannot always accurately predict the performance of each heuristic. To mitigate the potential impact of incorrect predictions, we propose the Confidence Stratification (ConS) mechanism. Other than the LLM-predicted fitness value $\xi_i$, ConS prompts an LLM to provide a corresponding confidence level $\phi_i \in [0,1]$ based on the degree of semantic similarity between $x_i$ and the most similar examples in $\mathcal{P}_e$. Subsequently, based on $\phi_i$, ConS selectively accepts the predicted fitness values of certain heuristics, while others are reevaluated using COP instances. Intuitively, we implement the following design. For heuristic $x_i$, if $\phi_i$ is sufficiently high, ConS deems $\xi_i$ accurate. If $\phi_i$ is moderately high, only the top-ranked candidates in this category should be trusted to directly adopt $\xi_i$ without reevaluation, reflecting the degraded confidence level. For low $\phi_i$ values, they can only be directly adopted if $\xi_i$ is greater than a predetermined threshold. Because for these heuristics with an acceptable yet sub-par performance score and a not-too-low confidence level, it is intuitive to deem them having inferior performance, without the need for precise predictions \citep{xu2021renas}. Specifically, we heuristically define this threshold by gauging the known prediction boundaries, i.e., $lb_t$ and $ub_t$. When $\phi_i$ is extremely low, $\xi_i$ is deemed unreliable and the corresponding heuristic must be reevaluated. Such design is implemented as follows to define the fitness function $f(x_i)$:
\begin{equation}
f(x_i) = 
\begin{cases}
    \xi_i , & \phi_i \geq 1-\delta, \\
    \xi_i, & 1- 2\delta \leq \phi_i < 1-\delta \; \land \; x_i \in \mathop{\operatorname{arg\,top(m_t)}}\limits_{x \in \mathcal{P}_c} \phi(x), \\
    \xi_i, &  1-3\delta \leq \phi_i < 1- 2\delta \; \land \; \xi_i > lb_t + 3\delta(ub_t - lb_t), \\
    \mathcal{F}(x_i), & \text{otherwise},
\end{cases}
\label{ConS}
\end{equation}
where $\delta \in [0, 1/3]$ denotes a predefined interval to distinguish the performance range of the produced heuristics (a smaller $\delta$ value means ConS only accepts the predicted scores with the highest confidence), $\mathcal{P}_c$ denotes the set of heuristics whose $\phi_i$ values lie within the $[1 - 2\delta, 1 - \delta$) interval, and $\mathcal{F}(\cdot)$ denotes the conventional fitness evaluation function, which uses COP instances to evaluate heuristics. Furthermore, we gradually decrease the number of heuristics that do not require reevaluation in $\mathcal{P}_c$ after each iteration. Specifically, we set an acceptance threshold $m_t = \lfloor \alpha \cdot \beta^{t} \cdot N_o \rfloor$, where $\alpha$, $\beta \in (0,1)$, and $N_o$ denotes the number of the produced heuristics in the current iteration. 


The pseudocode of Hercules-P is presented in Algorithm~\ref{alg}, and its code is available online (https://github.com/wuuu110/Hercules). The details about the adopted crossover and elitist mutation operators, along with other EC definitions, are presented in Appendix~\ref{appendixcm}.

\begin{table*}[!t]
    \caption{Performance comparison of different GLS algorithms on TSP}
     \vspace{-0.4cm}
    \label{tabgls}
    \centering
    \resizebox{0.75\textwidth}{!}{
    \begin{tabular}{c|c|c|c|c}
    \toprule

    \multirow{2}{*}{Algorithm} & \multicolumn{2}{c|}{LLM: Llama3-70b } & \multicolumn{2}{c}{ LLM: GPT-4o-mini}  \\
      & Gain (\%)  ($n=100$) & Gain (\%) ($n=200$) & Gain (\%) ($n=100$) & Gain (\%) ($n=200$) \\
    \midrule
    KGLS-Random     &  -137.13   &  0.47      &  \underline{63.64}    & 3.44   \\
    KGLS-EoH (ICML'24)   & -369.10  & 5.82  & 25.53 & 5.62  \\
    KGLS-ReEvo  (NeurIPS'24)   &  -661.69   & 2.19    &  -280.79  & 2.45       \\
    KGLS-Hercules-P (ours)     &  -218.91    & 4.71   & \textbf{71.05}  & \underline{7.46}  \\
    KGLS-Hercules (ours)     &   -12.48    & 3.42   &  42.98   & \textbf{11.10}    \\
    \bottomrule
    \end{tabular}}
     \vspace{-0.3cm}
    \end{table*}  




 \vspace{-0.4cm}
 \section{Experimental Results}
\label{4}

\begin{algorithm}[!t]
\caption{Hercules-P for Deriving Heuristics}
\label{alg}
\begin{flushleft}
\textbf{Input}: Maximum iteration number $T$\\
\textbf{Output}: Best heuristic $x_{\textit{best}}$ 
\end{flushleft}
\begin{algorithmic}[1]

\STATE  Initialize and evaluate population $\mathcal{P}$  \\ \textcolor{gray}{\# Omitting Steps 4, 12, and 13 makes Hercules-P fall back to the original Hercules algorithm }

\FOR{iteration $t=0$ to $T$}
    
    \STATE  Select parents set $\mathcal{P}_p$ via Eq. (\ref{rank}) \quad \textcolor{gray}{\textit{//{Rank-based selection}}} 

    \STATE  Select examples set $\mathcal{P}_e$ for PPP via Eq. (\ref{EXEMPLAR}) \quad \textcolor{gray}{\textit{//{EXEMPLAR}}}

\IF{$t \leq \lambda \cdot T$}
    \STATE Provide search directions using core components of elite heuristics \quad \textcolor{gray}{\textit{//{CAP}}}

  \ELSE  
  \STATE Provide search directions using core components of parent heuristics
\ENDIF

    \STATE        Derive heuristics using crossover based on the produced search directions

    \STATE         Derive heuristics using elitist mutation based on the produced search directions

    \STATE         Predict fitness of the newly produced heuristics \quad \textcolor{gray}{\textit{//{PPP}}}

    \STATE         Determine fitness values $f(\cdot)$ via Eq. (\ref{ConS}) \quad \textcolor{gray}{\textit{//{ConS}}}

    \STATE         Update $\mathcal{P}$ and $x_{\textit{best}}$ with new heuristics

\ENDFOR
\end{algorithmic}
\end{algorithm}
This section presents extensive experimental results on various HG tasks, COPs, and LLMs to assess the performance of both Hercules and Hercules-P. Please refer to Appendices~\ref{appendixD} to \ref{appendixF} for the experimental setups with predefined hyperparameter values, prompts used in this paper, additional experimental results, comparative examples of search directions produced by RP and CAP, and the produced heuristics, respectively.
 \vspace{-0.2cm}
\subsection{Deriving GLS Heuristics to Solve TSP}
\label{4.1}

In this subsection, we exploit Hercules and Hercules-P to derive penalty heuristics for Guided Local Search (GLS) to solve the Travelling Salesman Problem (TSP). The seed function is human-designed heuristic KGLS \citep{arnold2019_KGLS_VRP}. We choose three LLM-based HG algorithms as benchmarking models, namely Random, EoH \citep{liu2024evolution}, and ReEvo \citep{ye2024reevo}. Random is a straightforward method that derives heuristics directly using LLMs without incorporating search directions and is commonly used as a baseline model in NAS studies \citep{li2020random}. In addition, unless specified otherwise, for the performance of LLM-based HG algorithms, namely Random, EoH, ReEvo, Hercules-P, and Hercules, we report the average performance of three independent runs, following the prior study \citep{ye2024reevo}. The average gains of the heuristics produced by these algorithms are presented in Table~\ref{tabgls}, where $n$ denotes the problem scale. The gain measure is calculated as 1-(the performance of the LLM-produced heuristics)/(the performance of the original KGLS).



As shown in Table~\ref{tabgls}, for the 200-node TSP, the heuristics produced by Hercules using GPT-4o-mini outperform those produced by the other HG algorithms, yielding the best performance gain of 11.1\%. In addition, when GPT-4o-mini is adopted, the average gain of Hercules-P drops by only 3.64\% comparing to Hercules, securing the second-best performance. EoH ranks at the third place in the gain metric. The experimental results shown in Table~\ref{tabgls} highlight that the choice of LLM significantly impacts the performance of the produced heuristics. Nevertheless, Hercules and \mbox{Hercules-P}  consistently outperform ReEvo across all node scales, regardless of the LLM in use. In addition, to better illustrate the improvement achieved by adopting Hercules, we present a real-wolrd case study on optimizing a travel route to visit all U.S. state capitals in Appendix~\ref{case}. The results of the showcased case study highlight the substantial practical values of Hercules, offering a travel route that reduces the total travel distance by nearly 200 miles compared to those provided by other LLM-derived algorithms (yielding an improvement of 1.69\% in the gain metric over the second-best method).

\begin{table}[!t]
    \footnotesize
    \centering
    \caption{Search cost comparison of different LLM-based HG algorithms on TSP}
    \vspace{-0.4cm}
\label{tabsearchreport}
    \resizebox{1.05\linewidth}{!}{
    \begin{tabular}{c|c|c|c|c c}
         \cmidrule[\heavyrulewidth]{1-5}
         Algorithm &  Gain (\%) & Time (m) & Context  &  Generation & 
         \\
        \multicolumn{1}{l|}{} &  &  & Token (k) & Token (k) & 
         \multirow{6}{*}{\rotatebox[origin=c]{270}{GPT-4o-mini}}\\
        \cmidrule{1-5}
        KGLS-Random & 3.44\tiny$\pm$1.20 & \underline{28.5}\tiny$\pm$2.2 & \textbf{0.2} & \textbf{19.4} &\\
        KGLS-EoH (ICML'24) & 5.62\tiny$\pm$1.83  & 37.2\tiny$\pm$7.2 & \underline{43.5} &\underline{26.2} &\\
        KGLS-ReEvo (NeurIPS'24) & 2.45\tiny$\pm$10.93 & 37.7\tiny$\pm$12.2 & 95.5 & 42.0 &\\
        KGLS-Hercules-P (ours) & \underline{7.46}\tiny$\pm$5.36 & \textbf{23.6}\tiny$\pm$3.0 & 143.4 & 31.2 &\\
        KGLS-Hercules (ours) & \textbf{11.10}\tiny$\pm$0.69 & 30.6\tiny$\pm$1.4 &95.8 & 33.3 & \\ \cmidrule[\heavyrulewidth]{1-5}
    \end{tabular}
    }
    \vspace{-0.4cm}
\end{table}
\begin{table}[!t]
    \centering  
    \footnotesize
    \caption{Performance comparison of different constructive heuristic algorithms on TSPLIB} \label{tabselect}
    \vspace{-0.4cm}
    \resizebox{0.5\textwidth}{!}{
    \begin{tabular}{l|ccccc c} 
    \cmidrule[\heavyrulewidth]{1-6} 
        \multicolumn{1}{l|}{instances} & Random & EoH & \multicolumn{1}{c}{ReEvo} & Hercules-P & Hercules & \\ 
       \multicolumn{1}{l|}{(total number)}   &  &(ICML'24) & (NeurIPS'24) & (ours) &(ours) & \multirow{4}{*}{\rotatebox[origin=c]{270}{GPT-3.5-turbo}} \\ 
    \cmidrule{1-6} 
    $n<101$ (4) &    -3.92    &\textbf{16.68} & 1.18  & \underline{14.16} & 10.52 & \\
    $101 \leq n \leq  500$ (9)    &-3.80   &-0.60 & -1.17  & \underline{0.71} & \textbf{2.25} & \\
    $n>500$ (5)  &     -5.73   &\textbf{5.32} & 0.46  & 0.95 & \underline{5.18} & \\ \cmidrule{1-6}
    Avg. Gain (\%)   (18) &   -4.49   &\underline{4.80} & -0.16   & 3.42 & \textbf{4.87} & \\ 
    \cmidrule[\heavyrulewidth]{1-6} 
    \end{tabular}
    }
     \vspace{-0.6cm}
\end{table}
\begin{table*}[!t]
    \footnotesize
    \centering
    \caption{Performance comparison of different ACO algorithms on BPP and MKP}
    \label{tabaco}
    \vspace{-0.4cm}
    \resizebox{0.85\textwidth}{!}{
    \begin{tabular}{c|c|ccc|ccc}
    \toprule
     \multirow{2}{*}{Algorithm} & \multirow{2}{*}{Type} & \multicolumn{3}{c|}{BPP (Gain (\%)), LLM: Llama3.1-405b} & \multicolumn{3}{c}{MKP (Gain (\%)), LLM: Gemma2-27b} \\
     & & \(n=120\)  & \(n=500\) & \(n=1,000\) & \(n=120\)  & \(n=500\) & \(n=1,000\)  \\
     \midrule
     ACO+Random & ACO+LLM & 0.00 \tiny$\pm$0.00& -0.09$\pm$0.04 & 0.00\tiny$\pm$0.04 & 1.24\tiny$\pm$0.03 & 3.21\tiny$\pm$1.17 & 4.01\tiny$\pm$1.59 \\ 
     ACO+EoH (ICML'24) & ACO+LLM & 0.14\tiny$\pm$0.12  & 0.16\tiny$\pm$0.35 & 0.38\tiny$\pm$0.53  & \underline{1.61}\tiny$\pm$0.48 & 4.42\tiny$\pm$1.10 & 5.81\tiny$\pm$1.40 \\ 
     ACO+ReEvo (NeurIPS'24) & ACO+LLM & \underline{0.66}\tiny$\pm$0.50 & \underline{1.49}\tiny$\pm$0.25 & 2.01\tiny$\pm$0.34 & 1.59\tiny$\pm$0.72 & 4.67\tiny$\pm$0.95 & \underline{6.31}\tiny$\pm$0.38 \\ 
     ACO+Hercules-P (ours)& ACO+LLM  & 0.08\tiny$\pm$0.08 & 1.47\tiny$\pm$0.16 & \underline{2.04}\tiny$\pm$0.16 &1.44\tiny$\pm$0.38 & \underline{4.73}\tiny$\pm$0.90 & 6.14\tiny$\pm$1.21 \\ 
     ACO+Hercules (ours)& ACO+LLM & \textbf{0.84}\tiny$\pm$0.14 & \textbf{1.64}\tiny$\pm$0.17 & \textbf{2.19}\tiny$\pm$0.20 & \textbf{1.99}\tiny$\pm$0.50 & \textbf{6.40}\tiny$\pm$0.97 & \textbf{8.22}\tiny$\pm$1.17 \\     
    \bottomrule
    \end{tabular}}
    \vspace{-0.4cm}
\end{table*}  
\begin{table*}[!t]
    \footnotesize
    \centering
    \caption{Performance comparison of different NCO solvers on TSP and CVRP}
    \vspace{-0.4cm}
    \label{tabNCO}
    \resizebox{0.9\textwidth}{!}{
    \begin{tabular}{c|c|ccc|ccc}
    \toprule
     \multirow{2}{*}{Algorithm} & \multirow{2}{*}{Type} & \multicolumn{3}{c|}{TSP (Gain (\%))} & \multicolumn{3}{c}{CVRP (Gain (\%))}  \\
     & & \(n=200\)  & \(n=500\)  & \(n=1,000\) & \(n=200\) & \(n=500\) & \(n=1,000\)  \\
    \midrule 
     POMO+Random  &NCO+GPT-4o-mini  & \textbf{3.05} & -18.90 & -35.10& \textbf{3.07}& \underline{1.14} & \textbf{2.86} \\ 
     POMO+EoH (ICML'24) &NCO+GPT-4o-mini  & 2.19 & \underline{1.42} & \underline{1.47} & 0.48& -1.83 & 0.27 \\ 
     POMO+ReEvo (NeurIPS'24) &NCO+GPT-4o-mini  & 2.38 & -5.24 & -2.78 & 0.34 & -14.20 & -3.01 \\ 
     POMO+Hercules-P (ours) &NCO+GPT-4o-mini  &-0.10 & -4.81 & -3.58 & -0.57 & -3.29& -0.57 \\ 
     POMO+Hercules (ours) &NCO+GPT-4o-mini  & \underline{2.49} & \textbf{6.62} & \textbf{16.43} & \underline{1.53} & \textbf{1.22} & \underline{1.59}  \\ 
     \midrule 
     LEHD+Random  &NCO+GPT-4o-mini  & \underline{9.93} & \textbf{8.83} & 5.44 & 1.72& 2.33 & \underline{1.68} \\ 
     LEHD+EoH (ICML'24) &NCO+GPT-4o-mini  & \textbf{10.67} & \underline{7.73} & \underline{6.09} & 6.62 & 3.57& 0.47\\
     LEHD+ReEvo (NeurIPS'24) &NCO+GPT-4o-mini  & 6.94 & -1.78 & 1.56 &  \underline{10.19} & \underline{4.97} & 0.70 \\
     LEHD+Hercules-P (ours) &NCO+GPT-4o-mini &9.55 & 7.53 & \textbf{6.89} &4.44  &2.45 & 0.75 \\
     LEHD+Hercules (ours) &NCO+GPT-4o-mini  &7.46 & 6.64 & 5.14 & \textbf{14.37} &\textbf{7.90} & \textbf{2.33} \\
    \bottomrule
    \end{tabular}}
     \vspace{-0.4cm}
\end{table*}
Table~\ref{tabsearchreport} presents the search cost comparison of LLM-based HG algorithms across four metrics, namely gain (identical to the bottom-right cell of Table~\ref{tabgls}), search time, context token, and generation token. The results show that Hercules yields better gains without substantially increasing the costs of context and generation tokens, compared to ReEvo. Moreover, ReEvo and EoH spend longer search time when compared to the others, likely due to their ineffective search directions, which cause the LLM to derive complex but suboptimal heuristics. The std value of 10.93 for ReEvo further underscores this issue. On the other hand, \mbox{Hercules-P} reduces the overall search time to 77\% (23.6/30.6) of that required by Hercules. Although \mbox{Hercules-P} uses approximately 1.5 times more context tokens than Hercules and ReEvo, it does not significantly increase the cost of generation tokens, which are typically more expensive \citep{OpenAI}. This makes \mbox{Hercules-P} ideal for environments with limited computing resources. Notably, Random utilizes only 0.2k context tokens, because of its simple prompts used for heuristic generation. However, this simplicity limits its ability to derive well-performing heuristics.



 \vspace{-0.3cm}
\subsection{Constructive Heuristics to Solve TSP}
\label{4.2}

To assess the generalization capabilities of Hercules and Hercules-P across different HG tasks, we employ them in this subsection to derive constructive heuristics, which sequentially select unvisited nodes for solving real-world TSPLIB benchmarks \citep{reinelt1991tsplib}. The seed function is genetic programming hyper-heuristic \citep{duflo2019gp_hh_tsp}. As shown in Table~\ref{tabselect}, Hercules achieves the highest average gain of 4.87\% across eighteen TSPLIB instances, followed by EoH with the average gain of 4.8\%. In contrast, both Random and ReEvo perform poorly, yielding negative gains on average, i.e., failing to improve the performance of the seed function.


 \vspace{-0.4cm}
\subsection{Deriving Heuristic Measures for ACO to Solve BPP, MKP, OP, and TSP}
\label{4.3}

In this subsection, we exploit Hercules and Hercules-P to derive heuristic measures for Ant Colony Optimization (ACO) applied to the Bin Packing Problem (BPP) and Multiple Knapsack Problem (MKP). The seed function is a conventional ACO algorithm \citep{4129846}. We adopt Llama3.1-405b to solve BPP while adopting Gemma2-27b to solve MKP. This is because Llama3.1-405b fails to improve the seed function of MKP regardless of which LLM-based HG algorithm is executed. As shown in Table~\ref{tabaco}, Hercules outperforms the other algorithms across all COPs and LLMs, with particularly strong performance observed when solving the 1,000-scale MKP, achieving an 8.22\% gain. In addition, when using Llama3.1-405b, Random fails to derive superior heuristics compared to the original ACO, while EoH achieves only a modest improvement, falling short when compared to the substantial gains obtained by Hercules-P and Hercules. 
\begin{figure}[!t]
    \centering
    \subfloat[OP]{%
        \includegraphics[width=0.5\linewidth]{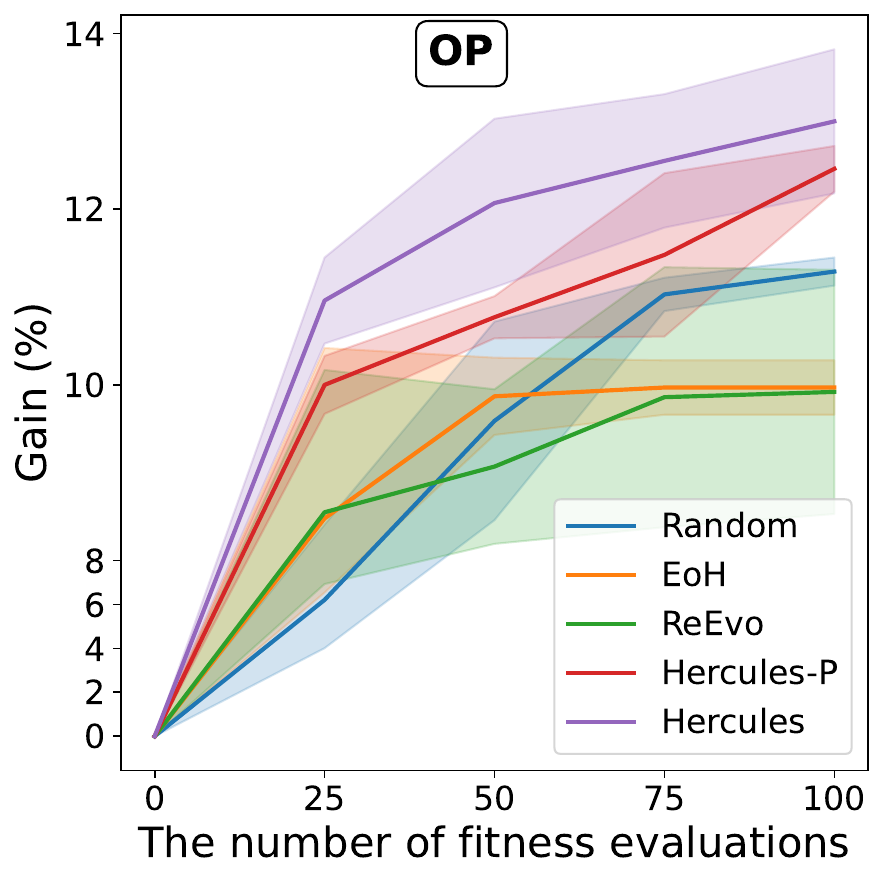}%
    }%
    \subfloat[TSP]{%
        \includegraphics[width=0.5\linewidth]{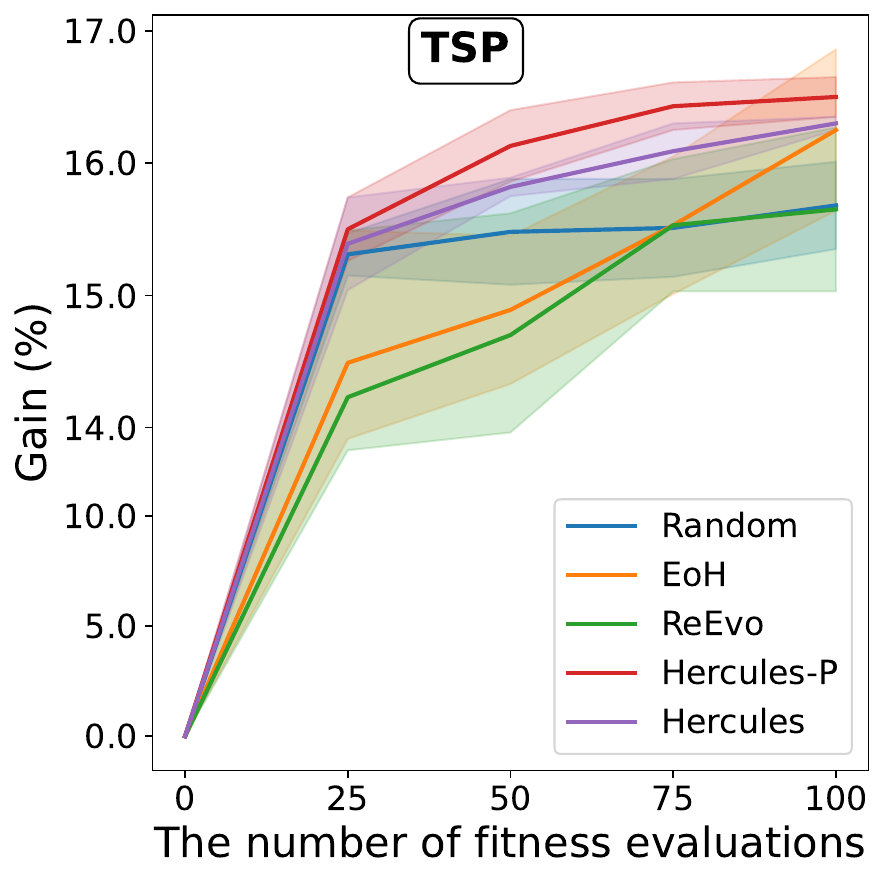}%
    }%
\caption{Convergence curves of different HG algorithms.}
    \label{figop}
    \vspace{-0.6cm}
\end{figure}
In addition, Figure~\ref{figop} presents the convergence curves of various LLM-based HG algorithms during their search processes for deriving ACO heuristic measures to solve the Orienteering Problem (OP) and TSP. The adopted LLM is GLM-4-Plus. As shown in Figure~\ref{figop}, Hercules and Hercules-P consistently derive better heuristics more efficiently than other LLM-based HG algorithms, during the entire search process. This is primarily attributed to the proposed CAP.

 \vspace{-0.2cm}
\subsection{Reshaping Attention Scores for NCO to Solve TSP and CVRP}
\label{4.4}
\begin{table*}[!t]
    \centering  
    \footnotesize
    \caption{Search time comparison of different LLM-based HG algorithms on diverse HG tasks} \label{tabtime}
    \vspace{-0.4cm}
\resizebox{0.75\textwidth}{!}{
    \begin{tabular}{c|l|ccccc} 
    \toprule
     \multirow{8}{*}{Time (m)}  & \multicolumn{1}{r|}{Algorithm } &  Random & EoH & \multicolumn{1}{c}{ReEvo} & Hercules-P & Hercules  \\ 
       & \multicolumn{1}{l|}{Task} &   &(ICML'24) & (NeurIPS'24) & (ours) &(ours)  \\ 
\cmidrule{2-7}
    &TSP-POMO    &   \underline{15.95}   & 18.17 & 17.89   & \textbf{11.50} & 22.12  \\ 
    &CVRP-POMO    &   \underline{16.86}   & 30.54 & 29.57   & \textbf{9.51} & 10.28  \\ 
    &TSP-LEHD    &   \underline{30.58}   & 39.55 & 37.25   & \textbf{28.72} & 41.43   \\ 
    &CVRP-LEHD ($n=200$) & \underline{45.73} & 67.27 & 61.58  & \textbf{31.20} & 42.80   \\ 
   & CVRP-LEHD ($n=500$) &\underline{149.31} & 224.01 & 215.61 & \textbf{110.28} & 178.01 \\ 
   & CVRP-LEHD ($n=1,000$) & \underline{639.83}& 854.25 & 854.71  &\textbf{310.98}& 757.67 \\ 
    \bottomrule
    \end{tabular}
    }
     \vspace{-0.4cm}
\end{table*}
\begin{table*}[!t]
    \footnotesize
    \centering
    \caption{Ablation study results on different design choices}
    \label{ablation}
    \vspace{-0.4cm}\resizebox{1\linewidth}{!}{
    \begin{tabular}{l|c||l|c||l|c||l|c c}
        \cmidrule[\heavyrulewidth]{1-8}
        Algorithm & Gain (\%) & Algorithm & Gain (\%) & Algorithm & Gain (\%) & Algorithm & Gain (\%) & \multirow{5}{*}{\rotatebox[origin=c]{270}{GPT-4o-mini}}\\ \cmidrule{1-8}
        w/o CAP &  3.12  &Hercules ($\lambda=0.5$) & 5.96 & w/o ConS & -4.06 & Hercules-P ($\delta=0.2$) & \underline{7.01} &  \\
        w/o rank-based selection & \underline{8.49} & Hercules ($\lambda=0.9$) & \underline{8.90} & w/o EXEMPLAR & \underline{-0.30} & Hercules-P ($\delta=0.3$)&6.21 &\\ 
         &  & Hercules ($\lambda=1$) & 5.60 & & & & &\\ \cmidrule{1-8}
        Hercules (w/o PPP) & \textbf{11.10} & Hercules ($\lambda=0.7$) & \textbf{11.10} & Hercules-P & \textbf{7.46} &Hercules-P ($\delta=0.1$) & \textbf{7.46} & \\ \cmidrule[\heavyrulewidth]{1-8}
    \end{tabular}
    }
    \vspace{-0.4cm}
\end{table*} 
Recently, \citet{wang2024distance} demonstrated that reshaping attention scores can enhance the generalization performance of NCO solvers trained on small-scale COPs for solving large-scale COPs. To assess the effectiveness of Hercules and Hercules-P on NCO solvers, following \citep{ye2024reevo}, we select DAR \citep{wang2024distance} as the seed function for TSP and the vanilla POMO \citep{kwon_pomo_2020} and LEHD \citep{luo_neural_2023} as seed functions for Capacitated Vehicle Routing Problem (CVRP). As shown in Table~\ref{tabNCO}, Random outperforms the other four LLM-based HG algorithms on certain tasks. A plausible reason for this is that the LLM corpora may lack sufficient knowledge of emerging NCO domains, thus limiting the performance of the other four LLM-based HG algorithms. Nevertheless, the heuristics derived by Hercules outperform the corresponding seed functions across a wider range of tasks compared to Random. For example, Hercules performs better than Random on the 500- and 1,000-node scales for the TSP-POMO task. In addition, in Table~\ref{tabtime}, we present the search time of different LLM-based HG algorithms across diverse NCO tasks. As shown in Table~\ref{tabtime}, Hercules-P outperforms the other LLM-based HG algorithms in terms of search time, while Random ranks at the second place. On these NCO tasks, Hercules-P reduce the search time by 48\%, 7\%, 31\%, 27\%, 38\%, and 59\%, respectively, when compared to Hercules. This reduction in search time is especially significant for large-scale COPs, where search can extend to several hours. In such cases, incorporating PPP demonstrates high efficiency in reducing resource expenditure.

 \vspace{-0.4cm}
\subsection{Ablation Studies}
\label{4.5}


In this subsection, we conduct ablation studies to investigate the effectiveness of the design choices of Hercules and Hercules-P, and present the results in Table~\ref{ablation}. The adopted HG task is deriving penalty heuristics for GLS to solve TSPs (see Section~\ref{4.1}). Specifically, w/o CAP refers to the setting using RP to provide search directions, w/o rank-based selection refers to the setting that randomly selects parent heuristics, w/o ConS refers to the setting that PPP assumes all predictions are accurate, and w/o EXEMPLAR refers to the setting that heuristic examples are randomly selected from the current population. For all the other experiments presented in this paper, $\lambda = 0.7$ is applied for Hercules, and $\delta = 0.1$ is applied for Hercules-P. As shown in Table~\ref{ablation}, when CAP is omitted, the gain decreases by 7.98\%, further demonstrating that CAP produces more specific search directions. In addition, the proposed rank-based selection mechanism significantly contributes to the superior performance of Hercules. For Hercules-P, ConS effectively determines unreliable predictions, preventing them from negatively affecting the derivation of high-performance heuristics. Finally, when EXEMPLAR is omitted, the gain decreases by 7.76\%, mainly due to the associated degradation in predictive accuracy (elaborated in the following paragraph). 


\begin{wrapfigure}[15]{r}{0.25\textwidth}
  \centering
  \includegraphics[width=0.24\textwidth]{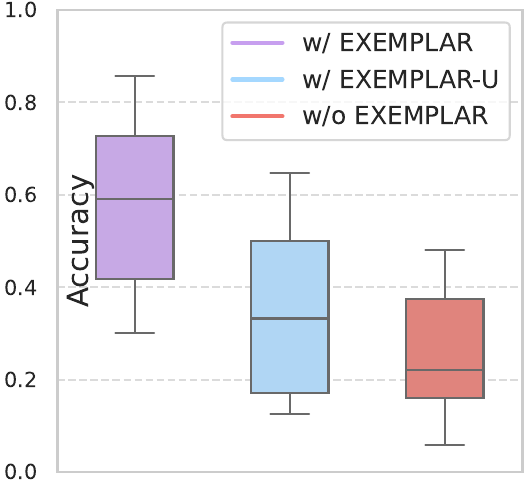}
  \caption{Ablation study on different EXEMPLAR variants.}

  \label{accuracy}
\end{wrapfigure}
We further present the predictive accuracy of PPP with and without EXEMPLAR, both of which are executed ten times, aiming to perform meaningful statistical tests. In addition, we include w/~EXEMPLAR-U as an additional setting, where w/ EXEMPLAR-U is able to select heuristics with identical fitness values. To assess whether different versions of EXEMPLAR can accurately predict the fitness values of the produced heuristics, we need to set a quantifying measure. Specifically, we intuitively deem a prediction accurate if the absolute error between the predicted fitness value and the true fitness value is less than $\delta\cdot(ub_t-lb_t)$. As shown in Figure~\ref{accuracy}, the inclusion of EXEMPLAR improves the median of predictive accuracy by 26\% and 37\% (both significantly different: $p= $0.048 and 0.004) when compared to w/~EXEMPLAR-U and w/o~EXEMPLAR, respectively. In addition, the Pearson correlation coefficient analysis reveals a correlation coefficient of 0.39, indicating a moderate linear relationship between the predicted and true values. The one-way ANOVA test results yield a $p$-value of 0.6, suggesting that the mean difference between the predicted and true values is not statistically significant. It is imperative to clarify that although the proposed PPP may seem less accurate in predicting heuristic performance, the values shown in Figure~\ref {accuracy} are determined by a strict measure of fitness values as afore-defined and they do not exhibit a strong correlation with the overall performance of Hercules-P, because many produced heuristics are reevaluated (see ConS in Section~\ref{3.2}). As discussed in Sections~\ref{4.1} and \ref{4.4}, Hercules-P reduces search time by 7\%$\sim$59\% when compared to Hercules, while achieving on-par gain. We strongly believe that PPP is highly beneficial for HG tasks that require rapid solutions, e.g., deriving heuristics for the dynamic, near-real-time allocation of resources in 5G mobile edge cloud networks \citep{9915432}. We plan to extend PPP by integrating it with other methods, such as beam search, to further enhance its predictive accuracy. 


\vspace{-0.2cm}
\section{Conclusion}
\vspace{-0.1cm}
To derive well-performing heuristics, we propose Hercules, which exploits our proprietary CAP to abstract the core components from elite heuristics, to produce more specific search directions. In addition, we introduce Hercules-P, a resource-efficient variant that integrates CAP with our novel PPP. PPP exploits previously evaluated heuristics to predict the performance of newly produced ones, thereby reducing the required computing resources for heuristic evaluations. The experimental results demonstrate the effectiveness of Hercules, Hercules-P, and all our designed mechanisms.
\vspace{-0.2cm}
\section*{Acknowledgments}
\vspace{-0.1cm}
This work was supported in part by the Jilin Provincial Department of Science and Technology Project under Grant 20230201083GX. We also thank the Computing Center of Jilin Province for providing technical support.
\bibliography{llm}
\bibliographystyle{ACM-Reference-Format}

\appendix

\section{Derivation of Information Gain Formula}
\label{appendixA}

\begin{proposition}
The information gain from abstracting core components is equal to:
\begin{align}
    IG(\Omega_t) = -\sum\nolimits_{j=0}^{k} p_j \log p_j \in (0, \log (k+1)].
\end{align}
\end{proposition}
\begin{proof}
\begin{align*}
    IG(\Omega_t) &= H(\Omega_t) - p_0 H(\Omega_0) - \dots - p_{k} H(\Omega_k) \\
    &= - \sum\nolimits_\_{i: \omega_i \in \Omega_t} p(\omega_i | \Omega_t) \log p(\omega_i | \Omega_t) \\
    &\quad + p_0 \sum\nolimits_{i: \omega_i \in \Omega_0} p(\omega_i | \Omega_0) \log p(\omega_i | \Omega_0) + \dots \\
    &\quad + p_k \sum\nolimits_{i: \omega_i \in \Omega_k} p(\omega_i | \Omega_k) \log p(\omega_i | \Omega_k) \\
   &= \sum\nolimits_{i: \omega_i \in \Omega_0} p(\omega_i | \Omega_0) \left[\log p(\omega_i | \Omega_0) - \log p(\omega_i | \Omega_t)\right] + \dots \\
   &\quad + \sum\nolimits_{i: \omega_i \in \Omega_k} p(\omega_i | \Omega_k) \left[\log p(\omega_i | \Omega_k) - \log p(\omega_i | \Omega_t)\right].
\end{align*}
According to the conditional probability, $p_j \cdot p(\omega_i | \Omega_j) = p(\omega_i | \Omega_t)$, $\forall j \in \{0, 1,\cdots, k\}$. Thus, the \( j \)th term simplifies to the following expression:
\begin{align*}
& \sum\nolimits_{i: \omega_i \in \Omega_j} p(\omega_i | \Omega_j) \left[\log p(\omega_i | \Omega_j) - \log p(\omega_i | \Omega_t)\right] \\
&= \sum\nolimits_{i: \omega_i \in \Omega_j} p(\omega_i | \Omega_j) \log \frac{p(\omega_i | \Omega_j)}{p(\omega_i | \Omega_t)} \\
&= -\sum\nolimits_{i: \omega_i \in \Omega_j} p(\omega_i | \Omega_j) \log p_j \\
&= -p_j \log p_j.
\end{align*}
Therefore, we conclude that
\begin{align}
    IG(\Omega_t) = -\sum\nolimits_{j=0}^{k} p_j \log p_j.
\end{align}
When \(\forall j \in \{0, 1,\cdots, k\}, p_j = \frac{1}{k+1} \), \( IG(\Omega_t) \) reaches its maximum value of \( \log(k+1) \). When \( \exists j \in \{0, 1,\cdots, k\} \) s.t. \( p_j = 1 \), \( IG(\Omega_t) \) reaches its minimum value of 0. However, due to the diverse nature of LLM training corpora, the LLM will not consistently provide the same direction. Therefore, by abstracting core components, the unspecificity (entropy) can decrease within the \( (0, \log(k+1)] \) interval.
\end{proof}
\vspace{-0.3cm} 
It is important to note that each heuristic may consist of multiple core components, or the same core component may be used across multiple heuristics. However, to simplify the deduction procedures, we assume a one-to-one correspondence between core components and heuristics in our analysis. This assumption does not affect the validity of our conclusion that $IG(\Omega_t) = -\sum\nolimits_{j=0}^{k} p_j \log p_j$, whose value ranges within the $ (0, \log (k+1)]$ interval, because the made assumption primarily serves to streamline the proof process. Specifically, if the number of core components $k_c$ differs from the number of heuristics $k$, $IG(\Omega_t)$ would be partitioned into $k_c + 1$ subsets instead of $k+1$ (the assumed case of a one-to-one correspondence between core components and heuristics). Despite this difference, the mathematical framework for entropy and information gain remains unchanged except the said replacement of $k$ with $k_c$. In addition, our framework is designed in a way that a single core component can map to multiple search directions, rather than having a single search direction corresponding to multiple core components. For example, the core component “Normalizing values” may correspond to directions such as “Normalize penalties relative to overall distance” or “Normalize distances effectively to balance contributions”. This assumption ensures flexibility in mapping core components to directions. Therefore, we deem our assumption of partitioning all possible directions into mutually exclusive subsets associated with core components is reasonable.
\section{Adopted Crossover, Elitist Mutation Operators, and Other EC Definitions}
\label{appendixcm}
For Hercules and Hercules-P, each heuristic code snippet denotes an individual within the population. Parent heuristics refer to the heuristics selected according to Eq. (\ref{rank}), which are utilized during the crossover and mutation processes to derive offspring heuristics. Elite heuristics denote the top-$k$ heuristics selected based on corresponding fitness values within the current population. During population initialization, we employ a simple prompt proposed by \citep{ye2024reevo} to guide the LLM in randomly deriving the initial population. For consistency, we adopt the crossover and mutation operators proposed by \citep{ye2024reevo} in all the experiments presented in this paper. Specifically, for the adopted crossover operator, two distinct parent heuristics are selected according to Eq. (\ref{rank}). Subsequently, the relative fitness values of these two heuristics determine which one serves as the primary learning exemplar for deriving an offspring heuristic. The employed mutation operator is elitist mutation, which derives multiple heuristics based on the historically best heuristic, aiming to produce high-performance ones. 


\vspace{-0.2cm}
\section{Detailed Hyper-parameters and Experimental Setups}
\label{appendixD}
\noindent \textit{\textbf{Hardware}} We comprehensively evaluate the performance of all algorithms, using a computer equipped with an Intel(R) Xeon(R) W-2235 CPU.
\begin{table}[!t]
    \centering
    \small
    \caption{Parameters of Hercules and Hercules-P}
    \vspace{-0.4cm}
    \label{tabparams}
    \begin{tabular}{l|l}
    \toprule
    Parameter & Value \\ \midrule
    LLM temperature & 1 \\
    CAP coefficients $k, \lambda$& 5, 0.7\\
    Maximum number of evaluations & 100 \\
    Population size $N$, Crossover rate, Mutation rate & 15, 1, 0.5 \\
    ConS coefficients $\delta, \alpha, \beta$& 0.1, 0.5, 0.8 \\
    \bottomrule
    \end{tabular}
     \vspace{-0.6cm}
\end{table}

\noindent \textit{\textbf{Hyper-parameters.}} In Table~\ref{tabparams}, we present the hyper-parameters of Hercules and Hercules-P. In addition, following the prior study \citep{ye2024reevo}, the temperature of the LLM is increased by 0.3 during the initial phase to enhance the diversity of the initial population. For the parameters of seed functions (e.g., KGLS parameters), we adopt the configurations specified in the prior study \citep{ye2024reevo} to ensure a fair comparison. This study also documents the definitions of all HG tasks used in this paper. In addition, following the prior study \citep{ye2024reevo}, the performance metric for TSP and CVRP is the gap, which is defined as the relative difference in the “average length” between corresponding heuristics and LKH3 \citep{lkh3}. For BPP and MKP, the performance metrics are the number of bins used and the total profit, respectively. For OP, the performance metric is the total prize collected by visiting nodes.

Finally, for all experiments in this paper, we exploit the training and test datasets to derive heuristics and assess the final derived heuristics, respectively. Specifically, during the search process, the performance of heuristics on the training datasets determines their fitness values. The heuristic with the best performance on the training dataset is selected as the final derived heuristic. We then further assess the performance of all final derived heuristics on test datasets and report the average experimental results in Section~\ref{4}. The details of training datasets and test datasets of all HG tasks can be found in the prior study \citep{ye2024reevo}.

\vspace{-0.2cm}
\section{Prompts Used in Hercules and Hercules-P}
\label{appendixE}
Prompts used for Hercules or Hercules-P can be categorized as problem-specific prompts (e.g., the heuristic description, COP description, seed function, and function signature) and general prompts (e.g., prompts for CAP and PPP). All these prompts are available in the provided source code link. 

\vspace{-0.2cm}
\section{Additional Experimental Results}
\subsection{Real-world Case Study}
\label{case}
\begin{table}[!t]
    \centering
    \small    \caption{Performance comparison of different GLS algorithms on a real-world TSP instance}
    \vspace{-0.4cm}
    \label{tabcase}
    \begin{tabular}{c|c|c}
    \toprule
    Algorithm &  Gain (\%) &  Length (mile) \\ \midrule
    KGLS-Random & 0.24 & 10836.79 \\
    KGLS-EoH (ICML'24) & \underline{0.38} & \underline{10821.49} \\
     KGLS-ReEvo  (NeurIPS'24)& 0.29  & 10831.75\\
     KGLS-Hercules-P (ours)&   0.17 & 10842.03 \\
     KGLS-Hercules (ours)& \textbf{2.07}  &\textbf{10637.36}
     \\
    \bottomrule
    \end{tabular}
     \vspace{-0.4cm}
\end{table}
To better illustrate the improved efficiency by adopting Hercules to derive heuristics in real-world scenarios, we examine a widely used real-world demonstration task of visiting all state capital cities in the United State (excluding Alaska and Hawaii) as a case study \citep{PADBERG19871}.  In this task, each capital city is treated as a node, framing the problem as a 49-node TSP instance.  In addition, in this section, we directly compare the KGLS algorithms derived from different LLM-based HG methods in Section~\ref{4.1}. As shown in Table~\ref{tabcase}, the Hercules-derived algorithms reduce the total mileage by nearly 200 miles compared to other algorithms, demonstrating the significant practical advantages of Hercules. 

\vspace{-0.4cm}
\subsection{Additional Experiments of Reshaping Attention Scores for NCO}
\label{appendixC3}

In this subsection, we adopt GLM-4-0520 as the LLM to further assess the performance of Hercules in reshaping attention scores for LEHD to solve large-scale TSP instances. In the experiments conducted in this subsection, we derive heuristics using training sets of problem sizes 200, 500, and 1,000 and evaluate their performance on test sets of the corresponding sizes. As shown in Table~\ref{tabNCOGLM}, Hercules achieves the best performance on test datasets with 200 and 500 nodes, while Hercules-P outperforms on the 1,000-node scale, achieving a gain of 11.72\% over the seed function. 

\begin{table}[!t]
    \footnotesize
    \centering
    \caption{Performance comparison of different LLM-based HG algorithms on TSP\_LEHD task}
     \vspace{-0.4cm}
    \label{tabNCOGLM}
    \resizebox{0.45\textwidth}{!}{
    \begin{tabular}{c|ccc c}
    \cmidrule[\heavyrulewidth]{1-4}
     \multirow{2}{*}{Algorithm} &  \multicolumn{3}{c}{TSP (Gain (\%))} & 
         \multirow{7}{*}{\rotatebox[origin=c]{270}{GLM-4-0520}}  \\
     &  \(n=200\)  & \(n=500\)  & \(n=1,000\) &  \\
     \cmidrule{1-4}
     LEHD+Random   & 8.48 & 8.36 & 7.70 & \\ 
     LEHD+EoH (ICML'24)  & \underline{10.84} & \underline{9.47} & 8.06& \\
     LEHD+ReEvo (NeurIPS'24)  & 10.13 & 8.70 & 6.97 & \\
     LEHD+Hercules-p (ours)  &9.98 & 8.80 & \textbf{11.72}  &\\
     LEHD+Hercules (ours)  &\textbf{11.06} & \textbf{9.24}& \underline{8.16} & \\
    \cmidrule[\heavyrulewidth]{1-4}
    \end{tabular}}
    \vspace{-0.4cm}
\end{table}

\vspace{-0.4cm}
\subsection{Performance of Different ACO Algorithms on CVRP under Black-box Settings}
Following the setup in the prior study \citep {ye2024reevo}, we assess the performance of different HG methods under black-box conditions, where no information about the COPs is provided to the LLM. We adopt GLM-4-Plus as the LLM and task it with deriving heuristic measures for ACO to solve CVRP. The experimental results demonstrate that Hercules outperforms the other HG methods. In addition, it is worth noting that Hercules performs well even when using a small-scale LLM (Qwen2.5-14B, distilled from DeepSeek-R1).
\label{blackbox}
\begin{table}[!t]
    \centering
    \small
    \caption{Performance comparison of different ACO algorithms on CVRP under black-box and white-box settings}
    \vspace{-0.4cm}
   \resizebox{0.4\textwidth}{!}{
    \begin{tabular}{c|c|cc}
    \toprule
     \multirow{2}{*}{Algorithm} & \multirow{2}{*}{LLM}&  \multicolumn{2}{c}{Gain (\%), $n$ = 50}   \\
     & & Black-box  & White-box     \\
     \midrule 
     ACO+Random  & GLM-4-Plus & 29.60 & 47.65  \\ 
     ACO+EoH (ICML'24) & GLM-4-Plus  & \underline{41.33} & \underline{48.43}  \\
     ACO+ReEvo (NeurIPS'24) & GLM-4-Plus  & 34.35 & 45.48   \\
     ACO+Hercules-p (ours) & GLM-4-Plus  &38.50 &  47.49  \\
     ACO+Hercules (ours) & GLM-4-Plus &\textbf{41.70 } & \textbf{48.92 }  \\
     ACO+Hercules (ours) & Qwen2.5-14B  &- & 44.99  \\
    \bottomrule
    \end{tabular}}
    \vspace{-0.4cm}
\end{table}

\vspace{-0.2cm}
\section{Search Directions Produced by RP and CAP}
\label{appendixB}
In this section, we present additional search directions produced by RP \citep{ye2024reevo} and CAP (our method) across various HG tasks, COPs, and LLMs. Additionally, all produced unspecific search directions are highlighted in blue. For example, GPT-4o-mini frequently suggests the term ``edge clustering", when performing RP. This direction "edge clustering" is frequently applied in tasks like recommender systems, where it helps identify patterns in user interactions and preferences. However, it is not commonly used in heuristic algorithms for solving COPs and is, therefore, considered unspecific.

\lstdefinestyle{directionstyle}{
    backgroundcolor=\color{gray!10},  
    basicstyle=\ttfamily\tiny,       
    frame=single,                     
    framerule=0.5pt,                  
    rulecolor=\color{black!30},       
    breaklines=true,                  
    breakatwhitespace=true,           
    showstringspaces=false,           
    columns=fixed,                 
    keepspaces=true
    captionpos=t,  
    emph={RP, CAP},               
    emphstyle=\bfseries\color{purple},       
    morekeywords={edge_clustering, symmetrize, route_clustering, the_opposite,symmetric_redundancy,successful_TSP_solutions,multi-vehicle_interactions,current_solution_state,edge_connectivity,experiment,
    historical_edge_frequencies,future_score,demand_forecasting,vehicle_utilization_metrics},  
    keywordstyle=\bfseries\color{blue},          
    abovecaptionskip=1em,             
    belowcaptionskip=1em,             
    numberstyle=\tiny\color{gray},    
}
\vspace{-0.4cm}
\renewcommand{\lstlistingname}{Direction}
\crefname{listing}{Direction}{Directions}
\begin{figure}[H] 
\begin{lstlisting}[caption={The produced search directions for deriving penalty heuristics to solve TSP},  label={lst:d1}, style=directionstyle]
# The LLM used to provide search directions is GPT-4o-mini.
RP:
Consider edge_clustering, incorporate historical_edge_frequencies, and adapt penalties dynamically based on current path exploration.

CAP:
Focus on relative edge scoring, incorporate multiple factors like connectivity and distance, and enhance normalization techniques.

# The LLM used to provide search directions is Llama-3-70b.
RP:
Normalize and symmetrize heuristics; consider the_opposite (not including an edge) for more effective penalties.

CAP:
Focus on relative edge costs (e.g., proximity concept) rather than absolute deviations from average distance.
\end{lstlisting}
\end{figure}

\begin{lstlisting}[caption={The produced search directions for deriving ACO heuristic measures to solve BPP},  label={lst:d3}, style=directionstyle]
# The LLM used to provide search directions is Llama3.1-405b.
RP:
Consider non-linear relationships between demand ratios and heuristics, and experiment with different sparsification thresholds for better performance.

CAP:
Simplification and normalization of demand values can lead to more effective heuristics, reducing computational complexity.
\end{lstlisting}

\begin{lstlisting}[caption={The produced search directions for reshaping attention scores of POMO to solve CVRP},  label={lst:d6}, style=directionstyle]
# The LLM used to provide search directions is GPT-4o-mini.
RP:
Incorporate route_clustering, demand distribution analysis, and consider multi-vehicle interactions for enhanced heuristics.

CAP:
Emphasize vectorization over loops for performance. Enhance demand penalties to better reflect capacity constraints. Normalize distances effectively to balance contributions.  
\end{lstlisting}
\begin{lstlisting}[caption={The produced search directions for reshaping attention scores of LEHD to solve TSP},  label={lst:d7}, style=directionstyle]
# The LLM used to provide search directions is GPT-4o-mini.
RP:
Incorporate edge_connectivity to prioritize clusters. Consider spatial locality using coordinates for refinement. Adaptively adjust weights based on current_solution_state.

CAP:
Use logarithmic scaling for distances, increase top-K selection, and implement normalization for better convergence and stability.
\end{lstlisting}

\section{Heuristics Derived by EoH and Hercules}
\label{appendixF}
In this section, we present the final derived heuristics for solving BPP derived by EoH and Hercules, repsectively. It is evident that, when Llama3.1-405b is adopted, EoH fails to derive intricate heuristics, which accounts for its poor performance in solving BPP (see Table~\ref{tabaco}).

\lstdefinestyle{eoh}{
    basicstyle=\tiny\ttfamily, 
    breaklines=true, 
    morekeywords={def, if, return,in, range,for,torch,log,mean,abs,heuristic,np,ndarray,zeros_like,Tensor,clone,unsqueeze,shape,
    sign, float,clamp,eye,tensor,device,median,largest,values,dim,keepdim,axis,topk,k,
    sum,int,zeros,abs,max,maximum,percentile,where,power,set,len,min,std,key,bool,inf,
    expand_dims,exp,outer,add},
    commentstyle=\itshape\color{gray},        
    keywordstyle=\bfseries\color{blue},       
    stringstyle=\color{teal},                 
    columns=flexible,
    captionpos=t
}

\renewcommand{\lstlistingname}{EoH}
\setcounter{lstlisting}{0}
\crefname{listing}{EoH}{EoH}

\begin{lstlisting}[caption={The ACO heuristic measure produced by EOH using Llama3.1-405b for solving BPP.},  label={lst: heuristic for EOHACO}, language=Python, style=eoh]

def EoH_1(demand: np.ndarray, capacity: int) -> np.ndarray:
    demand_ratio = demand / capacity
    return np.tile(np.power(demand_ratio, 2), (demand.shape[0], 1)) * (1 - demand_ratio[:, np.newaxis])    
\end{lstlisting}

However, despite using the same LLM, Hercules is able to generate more intricate heuristics, one of which is presented as follows:

\lstdefinestyle{heuristicstyle}{
    basicstyle=\tiny\ttfamily, 
    breaklines=true, 
    morekeywords={def, if, return,in, range,for,torch,log,mean,abs,heuristic,np,ndarray,zeros_like,Tensor,clone,unsqueeze,shape,
    sign, float,clamp,eye,tensor,device,median,largest,values,dim,keepdim,axis,topk,k,
    sum,int,zeros,abs,max,maximum,percentile,where,power,set,len,min,std,key,bool,inf,
    expand_dims,exp,outer,add},
    commentstyle=\itshape\color{gray},        
    keywordstyle=\bfseries\color{blue},       
    stringstyle=\color{teal},                 
    columns=flexible,
    captionpos=t
}

\renewcommand{\lstlistingname}{Heuristic}
\setcounter{lstlisting}{0}
\crefname{listing}{Heuristic}{Heuristic}

\begin{lstlisting}[caption={The ACO heuristic measure produced by Hercules using Llama3.1-405b for solving BPP.},  label={lst: heuristic for BPPACO}, language=Python, style=heuristicstyle]
def heuristic(demand: np.ndarray, capacity: int) -> np.ndarray:
    """
    This function calculates the heuristics for the Bin Packing Problem (BPP).
    Parameters:
    demand (np.ndarray): A 1D array representing the sizes of the items.
    capacity (int): The capacity of each bin.
    Returns:
    np.ndarray: A 2D array where heuristics[i][j] represents how promising it is to put item i and item j in the same bin.
    """
    # Calculate the complementarity of each pair of items
    # The complementarity is the difference between the capacity and the sum of the demands of the two items
    complementarity = capacity - np.add.outer(demand, demand)
    # Apply exponential decay to the complementarity values
    # This reduces the dominance of large values and emphasizes the importance of small values
    decayed_complementarity = np.exp(-complementarity / capacity)
    # Normalize the demand values to be between 0 and 1
    normalized_demand = demand / demand.max()
    # Calculate the heuristic value for each pair of items
    # The heuristic value is the product of the normalized demands and the decayed complementarity
    heuristics = np.outer(normalized_demand, normalized_demand) * decayed_complementarity
    # Sparsify the matrix by setting unpromising elements to zero
    # Here, we consider elements with a value less than 0.5 as unpromising
    heuristics[heuristics < 0.5] = 0
    
    return heuristics

\end{lstlisting}

\end{document}